\documentclass[accepted]{uai2026} 

\usepackage[american]{babel}
\usepackage[T1]{fontenc}

\usepackage{natbib}
\usepackage{amsmath,amssymb,amsthm,mathtools}
\usepackage{booktabs}
\usepackage{graphicx}
\usepackage{algorithm}
\usepackage{algpseudocode}
\usepackage[font=small,labelfont=bf]{caption}
\usepackage{placeins}
\usepackage{tikz}
\usetikzlibrary{positioning,calc,arrows.meta}
\definecolor{lightblue}{RGB}{173,216,230}
\definecolor{lightgreen}{RGB}{144,238,144}
\definecolor{lightred}{RGB}{255,182,193}
\definecolor{lightyellow}{RGB}{255,255,224}
\definecolor{warnred}{RGB}{200,50,50}
\definecolor{darkblue}{RGB}{0,0,139}
\definecolor{privacyblue}{RGB}{0,80,160}

\newcommand{\E}{\mathbb{E}}
\newcommand{\Var}{\mathrm{Var}}
\newcommand{\Cov}{\mathrm{Cov}}
\newcommand{\R}{\mathbb{R}}
\newcommand{\N}{\mathcal{N}}
\newcommand{\eps}{\varepsilon}
\newcommand{\del}{\delta}
\newcommand{\norm}[1]{\left\lVert #1 \right\rVert}
\newcommand{\ip}[2]{\left\langle #1, #2 \right\rangle}
\newcommand{\dto}{\overset{d}{\longrightarrow}}
\newcommand{\pto}{\overset{p}{\longrightarrow}}

\theoremstyle{definition}

\newtheorem{assumption}{Assumption}

\theoremstyle{plain}
\newtheorem{theorem}{Theorem}

\newtheorem{proposition}{Proposition}
\newtheorem{corollary}{Corollary}

\title{Noise-Calibrated Inference from Differentially Private Sufficient Statistics in Exponential Families}

\author[1]{Amir Asiaee}
\author[1]{Samhita Pal}
\affil[1]{%
    Department of Biostatistics, Vanderbilt University Medical Center, Nashville, TN, USA
}
\affil[ ]{\texttt{\{amir.asiaeetaheri, samhita.pal\}@vumc.org}}

\begin{document}
\maketitle
\pagestyle{numbered} 
\setlength{\abovedisplayskip}{2pt plus 1pt minus 1pt}
\setlength{\belowdisplayskip}{2pt plus 1pt minus 1pt}
\setlength{\abovedisplayshortskip}{2pt plus 1pt}
\setlength{\belowdisplayshortskip}{2pt plus 1pt}
\setlength{\textfloatsep}{2pt plus 1pt minus 1pt}
\setlength{\floatsep}{2pt plus 1pt minus 1pt}
\setlength{\intextsep}{2pt plus 1pt minus 1pt}
\titlespacing{\section}{0pt}{2pt plus 1pt minus 1pt}{1pt plus 1pt minus 1pt}
\titlespacing{\subsection}{0pt}{2pt plus 1pt minus 1pt}{1pt plus 1pt minus 1pt}
\titlespacing{\subsubsection}{0pt}{2pt plus 1pt minus 1pt}{1pt plus 1pt minus 1pt}
\titlespacing{\paragraph}{0pt}{0pt plus 1pt minus 1pt}{1pt}

\begin{abstract}
Many differentially private (DP) data release systems either output DP synthetic data and leave analysts to perform inference as usual, which can lead to severe miscalibration, or output a DP point estimate without a principled way to do uncertainty quantification.
This paper develops a clean and tractable middle ground for exponential families: release only DP sufficient statistics, then perform noise-calibrated likelihood-based inference and optional parametric synthetic data generation as post-processing.
Our contributions are: (1) a general recipe for approximate-DP release of clipped sufficient statistics under the Gaussian mechanism; (2) asymptotic normality, explicit variance inflation, and valid Wald-style confidence intervals for the plug-in DP MLE; (3) a noise-aware likelihood correction that is first-order equivalent to the plug-in but supports bootstrap-based intervals; and (4) a matching minimax lower bound showing the privacy distortion rate is unavoidable.
The resulting theory yields concrete design rules and a practical pipeline for releasing DP synthetic data with principled uncertainty quantification, validated on three exponential families and real census data.
\end{abstract}

\section{Introduction}\label{sec:intro}

Synthetic data is often promoted as a privacy-friendly alternative to sharing raw datasets.
However, the \emph{scientific} objective is rarely to obtain a dataset that merely ``looks realistic''; it is to enable valid inference about population parameters (standard errors, $p$-values, coverage) and, increasingly, valid causal conclusions.
Recent benchmarks show that many DP synthetic data methods preserve some low-dimensional marginals yet can fail badly for inferential validity and hypothesis testing when analysts treat the synthetic data as real \citep{tao2022benchmark,raisa2023noiseaware,perez2024mwutest}.
These failures are not surprising: differential privacy injects randomness that must be accounted for in uncertainty quantification.

This paper focuses on the most mathematically tractable and widely used regime where \emph{sufficient statistics exist}:
regular exponential family models.
In exponential families, likelihood-based inference depends on the dataset only through the empirical sufficient statistic $\bar S = n^{-1}\sum_i s(X_i)$.
This creates a natural and clean separation between privacy and inference:
(1) release a DP-noisy sufficient statistic $\widetilde{\bar S}$, then (2) perform noise-calibrated inference from $\widetilde{\bar S}$ alone.
Because any downstream computation---including parameter estimation, confidence intervals, and parametric synthetic data generation---is a deterministic function of the released statistic, it inherits the same DP guarantee by post-processing.

Although DP inference and DP synthetic data have each received attention \citep{ferrando2022bootstrap,bernstein2018dpbayes,mckenna2021nist,mckenna2022aim}, few works provide a unified treatment that (a) gives explicit asymptotic theory for inference from DP sufficient statistics, (b) quantifies the exact variance inflation due to privacy, and (c) connects parametric synthetic data generation to the same inferential framework.
We fill this gap for exponential families.

\subsection{Contributions and Paper Roadmap}

We formalize a pipeline:
\[
D=\{X_i\}_{i=1}^n
\quad\longrightarrow\quad
\widetilde{\bar S}
\quad\longrightarrow\quad
\widetilde{\theta}
\quad(\longrightarrow\ D^{\mathrm{syn}}).
\]

\begin{figure}[t]
\centering
\resizebox{\columnwidth}{!}{%
\begin{tikzpicture}[>=Stealth, node distance=0.8cm, every node/.style={font=\small}]
    \node[draw, rounded corners, fill=lightblue, minimum width=2cm, minimum height=0.9cm, align=center] (data) {Raw Data\\$D = (x_1,\ldots,x_n)$};
    \node[draw, rounded corners, fill=lightgreen, minimum width=2cm, minimum height=0.9cm, align=center, right=1.2cm of data] (suff) {Sufficient Stat\\$\bar{S} = \frac{1}{n}\sum s(x_i)$};
    \node[draw, rounded corners, fill=lightred, minimum width=2cm, minimum height=0.9cm, align=center, right=1.2cm of suff] (noise) {Add Noise\\$\widetilde{\bar S} = \bar{S} + Z$};
    \draw[ultra thick, warnred, dashed] ($(noise.east)+(0.6,1.3)$) -- ($(noise.east)+(0.6,-1.3)$);
    \node[warnred, rotate=90, font=\footnotesize\bfseries] at ($(noise.east)+(0.85, 0)$) {Privacy Wall};
    \node[draw, rounded corners, fill=lightyellow, minimum width=1.9cm, minimum height=0.7cm, align=center, right=1.7cm of noise, yshift=1cm] (mle) {Compute MLE\\$\hat\theta(\widetilde{\bar S})$};
    \node[draw, rounded corners, fill=lightyellow, minimum width=1.9cm, minimum height=0.7cm, align=center, right=1.7cm of noise] (ci) {Confidence\\Intervals};
    \node[draw, rounded corners, fill=lightyellow, minimum width=1.9cm, minimum height=0.7cm, align=center, right=1.7cm of noise, yshift=-1cm] (syn) {Synthetic\\Data};
    \draw[->, thick] (data) -- (suff);
    \draw[->, thick] (suff) -- (noise);
    \draw[->, thick, darkblue] (noise.east) ++(1.2,0) |- (mle.west);
    \draw[->, thick, darkblue] (noise.east) ++(1.2,0) -- (ci.west);
    \draw[->, thick, darkblue] (noise.east) ++(1.2,0) |- (syn.west);
    \node[above=0.3cm of noise, font=\footnotesize, privacyblue] {$(\eps,\del)$-DP};
    \node[below=1.5cm of noise, font=\footnotesize\bfseries, privacyblue, text width=5cm, align=center] {Everything right of the wall\\is free --- still $(\eps,\del)$-DP!};
\end{tikzpicture}%
}
\caption{Pipeline overview. The noisy sufficient statistic $\widetilde{\bar S}$ is the only DP-protected release; all downstream tasks inherit the same $(\eps,\del)$-DP guarantee by post-processing.}
\label{fig:pipeline}
\end{figure}
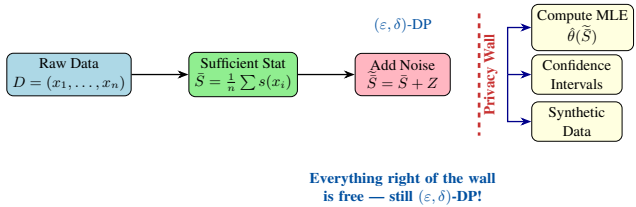

\begin{itemize}
  \item \textbf{DP mechanism:} release a noisy empirical sufficient statistic $\widetilde{\bar S}$ using the Gaussian mechanism \citep{dwork2006calibrating,dwork2014foundations}.
  \item \textbf{Inference:} compute either a plug-in DP MLE $\widetilde{\theta}_{\mathrm{plug}}$ or a noise-aware estimator $\widetilde{\theta}_{\mathrm{NA}}$.
  \item \textbf{Optional synthetic data:} sample $D^{\mathrm{syn}} \sim p(\cdot \mid \widetilde{\theta})$; by post-processing, the synthetic data inherits DP.
\end{itemize}
We provide (i) explicit asymptotic variance inflation, (ii) valid confidence intervals, and (iii) a lower bound in a canonical subclass that matches the upper bound rate.

\section{Related Work}\label{sec:related}

\paragraph{DP estimation and minimax rates.}
The foundational theory of DP estimation establishes that privacy imposes a cost on statistical accuracy.
\citet{dwork2006calibrating} introduced the Laplace and Gaussian mechanisms, calibrating noise to the sensitivity of a query.
Early work connecting DP to classical statistical inference includes \citet{wasserman2010statistical}.
For mean estimation of bounded data, \citet{smith2011privacy} showed that $(\eps,\del)$-DP estimators can achieve rates close to those of non-private estimators when $n$ is large relative to $1/\eps$, while lower bounds due to \citet{duchi2018minimax} demonstrate that an $\Omega(1/(n\eps))$ RMSE distortion is unavoidable in the local DP model and also appears in the central model for sufficiently small $\eps$.
In the central model, \citet{cai2021cost} established tight minimax rates for parameter estimation under $(\eps,\del)$-DP, showing that the cost of privacy in $\ell_2$ risk is $\Theta(d/(n^2\eps^2))$ for $d$-dimensional bounded problems; our Theorem~\ref{thm:lb} recovers the $d=1$ case.
For settings that rely heavily on Gaussian noise, alternative privacy accounting frameworks such as Gaussian differential privacy \citep{dong2022gdp} provide a complementary perspective.

\paragraph{DP inference and confidence intervals.}
A central challenge is constructing valid confidence intervals when estimators are perturbed for privacy.
General-purpose approaches include parametric bootstrap methods for DP confidence intervals \citep{ferrando2022bootstrap} and DP hypothesis tests for linear regression \citep{alabi2022hypothesis}.
Finite-sample DP confidence intervals for the Gaussian mean were studied by \citet{karwa2018finite}.
Asymptotic inference for DP M-estimators via noisy optimization was developed by \citet{avellamedina2023noisyopt}.
Simulation-based, finite-sample inference procedures for privatized data were studied by \citet{awan2025simbainference}.
Noise-aware Bayesian inference for exponential families that accounts for DP noise in sufficient statistics was developed by \citet{bernstein2018dpbayes}, who showed how to incorporate the known noise distribution into a conjugate posterior update.
More recently, \citet{wang2025dpbootstrap} developed a general DP bootstrap framework with privacy-aware variance estimation, and \citet{covington2025unbiased} provided unbiased DP estimators with valid confidence intervals for a broad class of statistical functionals.
\citet{awan2018dpuf} studied DP inference for the binomial proportion and related one-parameter families via optimal noise mechanisms that achieve exact finite-sample inference.
In graph exponential-family models, \citet{karwa2016noisydegrees} studied inference from DP-noisy sufficient statistics (noisy degrees) and synthetic graph generation.

\paragraph{DP synthetic data generation.}
Large-scale empirical benchmarks show that marginal-based / workload-based methods can outperform GAN-based methods for tabular synthesis, but inferential validity is not automatic \citep{tao2022benchmark}.
State-of-the-art synthesis frameworks include Private-PGM and its MST/NIST-MST instantiations \citep{mckenna2021nist}, the workload-adaptive AIM mechanism \citep{mckenna2022aim}, and more recent approaches that incorporate geometry-aware losses and optimal transport \citep{donhauser2024privpgd}.
A complementary line studies how to \emph{analyze} DP synthetic data with correct uncertainty, including noise-aware synthetic data generation plus multiple imputation \citep{raisa2023noiseaware} and empirical studies warning about inflated type-I error when synthetic data are analyzed naively \citep{perez2024mwutest}.

\paragraph{Positioning.}
Our work is closest in spirit to \citet{bernstein2018dpbayes} but differs in two ways.
First, we provide a complete frequentist asymptotic theory---including explicit variance inflation formulas, Wald confidence intervals, and a matching lower bound---that complements their Bayesian treatment.
Second, we show how the same released DP sufficient statistic can serve dual purposes: frequentist inference and parametric synthetic data generation, with both inheriting the same privacy guarantee.
This unifies the ``DP inference'' and ``DP synthetic data'' literatures in the exponential-family setting.
We note that the individual theoretical components---delta-method CLT for DP estimators and Le~Cam lower bounds---are well known in the DP statistics literature; our contribution is their integration into a single, practical pipeline with thorough empirical validation and a clear demonstration that naive synthetic-data analysis fails.

\section{Setup}\label{sec:setup}

\subsection{Exponential Family Models}

Let $X \in \mathcal{X}$ and consider a (minimal) regular exponential family with natural parameter $\theta \in \Theta \subset \R^d$:
\begin{equation}
p_\theta(x)=h(x)\exp\left(\ip{\theta}{s(x)}-A(\theta)\right),
\end{equation}
where $s(x)\in\R^d$ is the sufficient statistic and $A(\theta)$ is the log-partition function.
Given i.i.d.\ data $D=\{X_i\}_{i=1}^n$, the empirical sufficient statistic is
\[
\bar S(D)=\frac{1}{n}\sum_{i=1}^n s(X_i)\in\R^d.
\]
The (non-private) maximum likelihood estimator (MLE) $\hat\theta$ satisfies $\nabla A(\hat\theta)=\bar S(D)$ when the MLE exists.

\subsection{Differential Privacy and the Gaussian Mechanism}

Two datasets $D,D'$ are \emph{neighbors} if they differ in a single record.
A randomized mechanism $\mathcal{M}$ is $(\eps,\del)$-DP if for all neighboring $D,D'$ and measurable sets $E$,
\[
\Pr(\mathcal{M}(D)\in E)\le e^\eps \Pr(\mathcal{M}(D')\in E)+\del.
\]
We use the Gaussian mechanism \citep{dwork2006calibrating,dwork2014foundations}.

\begin{assumption}[Bounded sufficient statistics]
\label{assump:bounded}
There exists $B>0$ such that $\norm{s(x)}_2\le B$ for all $x\in\mathcal{X}$.
\end{assumption}

Under Assumption~\ref{assump:bounded}, the $\ell_2$ sensitivity of $\bar S(D)$ is at most $\Delta_2 = 2B/n$.

\section{Mechanisms and Estimators}\label{sec:mechanisms}

\subsection{DP Sufficient Statistic Release}

\begin{algorithm}[t]
\caption{DP Sufficient Statistic Release and Inference (Exponential Families)}
\label{alg:dp_suffstat}
\begin{algorithmic}[1]
\Require Dataset $D=\{X_i\}_{i=1}^n$, bound $B$, privacy $(\eps,\del)$
\State Compute $\bar S = \frac{1}{n}\sum_{i=1}^n s(X_i)$ \Comment{sufficient statistic}
\State Set $\Delta_2 \leftarrow 2B/n$ \Comment{$\ell_2$ sensitivity of $\bar S$}
\State Set $\sigma \leftarrow \mathrm{AGM}(\Delta_2,\eps,\del)$ \Comment{analytic Gaussian calibration \citep{balle2018improving}}
\State Sample $Z\sim \N(0,\sigma^2 I_d)$ and release $\widetilde{\bar S}=\bar S+Z$
\State \textbf{Plug-in:} $\widetilde{\theta}_{\mathrm{plug}} \leftarrow (\nabla A)^{-1}(\widetilde{\bar S})$
\State \textbf{Noise-aware:} $\widetilde{\theta}_{\mathrm{NA}} \leftarrow \arg\max_{\theta\in\Theta} \log p(\widetilde{\bar S}\mid \theta)$ (Section~\ref{sec:noiseaware})
\State (Optional) Sample synthetic data $D^{\mathrm{syn}}\sim p(\cdot\mid \widetilde{\theta}_{\mathrm{plug}})$ (or $\widetilde{\theta}_{\mathrm{NA}}$)
\end{algorithmic}
\end{algorithm}

Here $\mathrm{AGM}(\Delta_2,\eps,\del)$ denotes the analytic Gaussian mechanism calibration that returns the smallest $\sigma$ such that adding $Z\sim\N(0,\sigma^2 I_d)$ to a query with $\ell_2$ sensitivity $\Delta_2$ is $(\eps,\del)$-DP \citep{balle2018improving}.
Throughout, we set $\del = 1/n^2$ for all experiments; the analytic Gaussian mechanism calibration \citep{balle2018improving} then determines the minimal $\sigma$.

\begin{theorem}[Privacy of sufficient-statistic release]
\label{thm:dp}
Under Assumption~\ref{assump:bounded}, releasing $\widetilde{\bar S}=\bar S(D)+Z$ with $Z\sim \N(0,\sigma^2 I_d)$ and $\sigma$ as in Algorithm~\ref{alg:dp_suffstat} is $(\eps,\del)$-DP.
Moreover, any (randomized) post-processing of $\widetilde{\bar S}$ (including $\widetilde{\theta}$ and $D^{\mathrm{syn}}$) is also $(\eps,\del)$-DP.
\end{theorem}

\paragraph{Proof sketch.}
The $\ell_2$ sensitivity of $\bar S$ is at most $2B/n$.
The stated $\sigma$ is chosen by the analytic Gaussian mechanism calibration, which gives an $(\eps,\del)$-DP guarantee for the Gaussian mechanism for any $\eps>0$ \citep{balle2018improving}; DP then follows from the Gaussian mechanism guarantee and post-processing invariance \citep{dwork2006calibrating,dwork2014foundations}.
\qed

\subsection{Plug-in DP MLE and Asymptotic Normality}

The plug-in estimator solves $\nabla A(\widetilde{\theta}_{\mathrm{plug}})=\widetilde{\bar S}$.
In a regular exponential family, $\nabla^2 A(\theta)=I(\theta)$ equals the Fisher information for one sample.

\begin{assumption}[Regularity]
\label{assump:regular}
The model is a minimal regular exponential family with $\theta_0$ in the interior of $\Theta$.
The Fisher information $I(\theta_0)$ is positive definite.
\end{assumption}

\begin{theorem}[Asymptotic distribution and variance inflation]
\label{thm:clt}
Under Assumptions~\ref{assump:bounded}--\ref{assump:regular}, the plug-in DP MLE satisfies
\[
\sqrt{n}\left(\widetilde{\theta}_{\mathrm{plug}}-\theta_0\right)
\dto
\N\!\left(0,\ I(\theta_0)^{-1} \;+\; n\sigma^2\, I(\theta_0)^{-2}\right),
\]
where $\sigma^2$ is the variance of the Gaussian noise in Algorithm~\ref{alg:dp_suffstat} and $I(\theta_0)^{-2}:=I(\theta_0)^{-1}I(\theta_0)^{-1}$.
Equivalently,
\[
\widetilde{\theta}_{\mathrm{plug}}-\theta_0
=
\underbrace{O_p(n^{-1/2})}_{\text{sampling}}
+
\underbrace{O_p(\sigma)}_{\text{privacy}}.
\]
When $\sigma$ is calibrated for $(\eps,\del)$-DP with $\ell_2$ sensitivity $\Delta_2=2B/n$, the noise scale is proportional to $\Delta_2$ and decreases with $\eps$ (up to $\log(1/\del)$ factors), yielding the familiar $O_p((n\eps)^{-1})$ privacy distortion rate for bounded problems.
\end{theorem}

\paragraph{Proof sketch.}
The proof decomposes the error into a sampling term and a privacy term that are independent.
\emph{(i)} The multivariate central limit theorem (CLT) gives $\sqrt{n}(\bar S - \mu(\theta_0)) \dto \N(0, I(\theta_0))$, and the delta method applied through the inverse mean-parameter map $(\nabla A)^{-1}$ yields the classical $I(\theta_0)^{-1}/n$ variance.
\emph{(ii)} A first-order Taylor expansion of $(\nabla A)^{-1}$ around $\bar S$ gives $\widetilde{\theta}_{\mathrm{plug}} - \hat\theta = I(\hat\theta)^{-1} Z + O_p(\norm{Z}^2)$, contributing the $\sigma^2 I(\theta_0)^{-2}$ privacy variance.
\emph{(iii)} Since the Gaussian noise $Z$ is independent of the data, the two variance components add.
The full proof is in Appendix~\ref{app:proof_thm2}.
\qed

\begin{corollary}[When does privacy preserve classical efficiency?]
\label{cor:efficiency}
Under the conditions of Theorem~\ref{thm:clt}, if $n\eps^2 \to \infty$ (equivalently $\eps \gg n^{-1/2}$ up to log factors in $\del$), then the privacy-induced variance vanishes and $\widetilde{\theta}_{\mathrm{plug}}$ is asymptotically as efficient as the non-private MLE.
If $n\eps^2$ is bounded, DP noise contributes a non-negligible variance inflation.
\end{corollary}

\subsection{Noise-Aware Likelihood Correction}
\label{sec:noiseaware}

The plug-in estimator is first-order correct under smoothness, but in practice we often clip or otherwise sanitize $s(x)$ to bound sensitivity, inducing bias.
Noise-aware inference explicitly models the DP noise and (optionally) the clipping.

A simple approach is to define the \emph{noisy statistic likelihood}:
\begin{equation}
\label{eq:noisy_lik}
\ell_{\mathrm{NA}}(\theta;\widetilde{\bar S})
:= \log p(\widetilde{\bar S}\mid \theta),
\end{equation}
where $p(\widetilde{\bar S}\mid\theta)$ is induced by the data model and the DP mechanism.
In general, this is an intractable convolution, but it becomes tractable in two regimes:
\begin{itemize}
  \item \textbf{CLT regime:} Approximate $\bar S\mid\theta \approx \N(\nabla A(\theta), I(\theta)/n)$.
  Then $\widetilde{\bar S}\mid\theta \approx \N(\nabla A(\theta),\ I(\theta)/n + \sigma^2 I)$.
  \item \textbf{Conjugate / exact regimes:} For certain families and priors, exact noise-aware Bayes is available \citep{bernstein2018dpbayes}.
\end{itemize}

In the CLT regime, the noise-aware estimator solves a generalized least squares problem:
{\small
\begin{multline*}
\widetilde{\theta}_{\mathrm{NA}}
=
\arg\min_{\theta\in\Theta}\;
\bigl(\widetilde{\bar S}-\nabla A(\theta)\bigr)^{\!\top}\\
\bigl(\tfrac{1}{n}I(\theta)+\sigma^2 I\bigr)^{-1}
\bigl(\widetilde{\bar S}-\nabla A(\theta)\bigr).
\end{multline*}}

\paragraph{Numerical stabilization.}
In finite samples (especially under strong privacy), the estimated Fisher information $I(\widetilde\theta)$ can be ill-conditioned and the noise-aware objective can become flat.
In our implementation we apply the following numerical safeguards:
(i) Tikhonov regularization $I(\theta)\leftarrow I(\theta)+\lambda I$ with $\lambda = \max(10^{-6},\; 0.01\,\sigma^2)$;
(ii) box constraints $\theta_j \in [-10, 10]$ for the L-BFGS-B optimizer;
(iii) a quadratic anchor penalty $0.1\,\sigma^2 \sum_j (\theta_j - \widetilde\theta_{\mathrm{plug},j})^2$ in the noise-aware objective;
and (iv) clipping of variance diagonal entries at $10^6/n$.
These choices are purely for stability and do not change the first-order asymptotic theory.

\begin{proposition}[First-order equivalence]
\label{prop:equiv}
Under Theorem~\ref{thm:clt}'s conditions and with smoothness of $I(\theta)$, the noise-aware estimator $\widetilde{\theta}_{\mathrm{NA}}$ is first-order equivalent to the plug-in estimator:
\[
\widetilde{\theta}_{\mathrm{NA}}-\widetilde{\theta}_{\mathrm{plug}} = o_p(n^{-1/2}) + o_p(\sigma),
\]
and thus has the same asymptotic distribution as in Theorem~\ref{thm:clt}.
\end{proposition}

\paragraph{Proof sketch.}
At the plug-in solution, $\widetilde{\bar S} - \nabla A(\widetilde{\theta}_{\mathrm{plug}}) = 0$, which also zeros the noise-aware first-order condition.
Hence both estimators solve the same equation at leading order; the difference arises only through higher-order terms involving $\Sigma(\widetilde{\theta}_{\mathrm{plug}})^{-1} - \Sigma(\theta_0)^{-1}$, which is $O_p(n^{-1/2} + \sigma)$.
See Appendix~\ref{app:proof_prop1} for details.
\qed

\subsection{Confidence Intervals and Hypothesis Testing}

Theorem~\ref{thm:clt} suggests a Wald interval for each coordinate $\theta_j$:
\[
\mathrm{CI}_j = \left[\widetilde{\theta}_j \pm z_{1-\alpha/2}\sqrt{\widehat{\Var}(\widetilde{\theta})_{jj}}\right],
\]
with variance estimate
\[
\widehat{\Var}(\widetilde{\theta})
=
\frac{1}{n}\widehat{I}^{-1} + \sigma^2 \widehat{I}^{-2},
\quad
\widehat{I}:=I(\widetilde{\theta}).
\]
where $\widehat{\Var}(\widetilde{\theta})_{jj}$ denotes the $j$-th diagonal entry of $\widehat{\Var}(\widetilde{\theta})$.
Alternatively, one can use a DP parametric bootstrap to better capture finite-sample effects and clipping bias \citep{ferrando2022bootstrap}.
The bootstrap draws are generated from the noise-aware distribution $p(\widetilde{\bar S} \mid \theta)$, so the noise-aware likelihood model directly enables the bootstrap procedure; see Appendix~\ref{app:experiment_details} for implementation details.

\section{Lower Bounds}\label{sec:lower}

Upper bounds are only useful if they are close to optimal.
We provide a simple lower bound that matches the \emph{privacy term} rate.

\begin{theorem}[Unavoidable $\Omega(1/(n^2\eps^2))$ MSE lower bound]
\label{thm:lb}
Consider the one-dimensional exponential family on $\mathcal{X}=\{-1,+1\}$ with $s(x)=x$ and natural parameter $\theta\in\Theta\subset\R$.
For any $(\eps,\del)$-DP estimator $\widehat{\theta}$ based on $n$ i.i.d.\ samples with $\del < 1/4$ and $\eps\in(0,1]$,
\[
\inf_{\widehat{\theta}\ \text{$(\eps,\del)$-DP}} \ \sup_{\theta\in\Theta}
\E_\theta\!\left[(\widehat{\theta}-\theta)^2\right]
\;\ge\;
\frac{c}{n^2\eps^2},
\]
for a universal constant $c>0$.
In particular, the minimax RMSE is at least $\sqrt{c}/(n\eps)$.
Combined with Theorem~\ref{thm:clt}, this shows that the $1/(n\eps)$ privacy distortion rate (up to log factors in $\del$) is minimax optimal in this canonical subclass.
\end{theorem}

\paragraph{Proof sketch.}
Write the mean parameter as $\mu(\theta)=\tanh(\theta)$.
Given any DP estimator $\widehat\theta$, the post-processed estimator $\widehat\mu=\tanh(\widehat\theta)$ is also DP and satisfies $|\widehat\mu-\mu(\theta)|\le|\widehat\theta-\theta|$ since $\tanh$ is 1-Lipschitz.
Thus, any minimax lower bound for DP mean estimation transfers immediately to $\theta$.
The stated rate follows by specializing the central-model minimax lower bounds for bounded mean estimation under $(\eps,\del)$-DP due to \citet{cai2021cost}.
See Appendix~\ref{app:proof_thm3} for the full argument.
\qed

\section{Experiments}\label{sec:experiments}

The experiments are designed to validate three claims: (i) the asymptotic variance formula in Theorem~\ref{thm:clt} accurately predicts finite-sample behavior, (ii) Wald confidence intervals achieve nominal coverage, and (iii) plug-in and noise-aware estimators are compared under clipping.

\subsection{Experimental Design}

All experiments use $(\eps,\del)$-DP with $\del = 1/n^2$, a standard choice that makes the $\del$-failure probability negligible relative to the sample.

\paragraph{Models and sufficient statistics.}
We consider three canonical exponential families.
\begin{enumerate}
  \item \textbf{Gaussian mean estimation.}
  Data $X_i \sim \N(\mu, \sigma_0^2)$ with known $\sigma_0^2$.
  Sufficient statistic: $s(x) = x$.
  Natural parameter: $\theta = \mu/\sigma_0^2$.
  Clipping: truncate observations to $[-B, B]$.

  \item \textbf{Logistic regression.}
  Binary outcomes $(X_i, Y_i)$ with $\Pr(Y=1 \mid X) = \mathrm{sigmoid}(\theta^\top X)$.
  Sufficient statistic: $s(x,y) = yx \in \R^d$.
  Clipping: project features to $\norm{X}_2 \le B$.

  \item \textbf{Poisson regression.}
  Count outcomes $Y_i \mid X_i \sim \mathrm{Poisson}(\exp(\theta^\top X_i))$.
  Sufficient statistic: $s(x,y) = yx$.
  Clipping: truncate $Y$ to $[0, B_Y]$ and project features to $\norm{X}_2 \le B_X$.
\end{enumerate}

\paragraph{Methods compared.}
\begin{itemize}
  \item \textbf{Non-private MLE} (oracle): standard MLE with classical asymptotic CIs.
  \item \textbf{DP plug-in (Wald)}: Algorithm~\ref{alg:dp_suffstat} with Wald CIs using the variance formula from Theorem~\ref{thm:clt}.
  \item \textbf{DP noise-aware (Wald)}: the estimator from Section~\ref{sec:noiseaware} with the same CI formula.
  \item \textbf{DP parametric bootstrap}: the bootstrap procedure of \citet{ferrando2022bootstrap} adapted to sufficient-statistic release, where bootstrap samples are drawn from the noisy statistic's distribution.
  \item \textbf{Naive DP synthetic}: generate $D^{\mathrm{syn}} \sim p(\cdot \mid \widetilde{\theta}_{\mathrm{plug}})$, then apply standard (non-DP-aware) MLE and CIs to $D^{\mathrm{syn}}$.
\end{itemize}

\paragraph{Metrics.}
\begin{itemize}
  \item \textbf{Coverage}: empirical coverage of nominal 95\% confidence intervals across 1000 independent runs (each with fresh data and DP noise).
  \item \textbf{Interval length}: average CI width, measuring inferential precision.
  \item \textbf{Parameter MSE}: $\E[\norm{\widetilde\theta - \theta_0}^2]$, estimated over replications.
  \item \textbf{Type-I error}: rejection rate of $H_0: \theta_j = \theta_{0,j}$ at the true parameter.
  \item \textbf{Power}: rejection rate at a sequence of alternatives $\theta_j = \theta_{0,j} + \Delta$ for $\Delta \in \{0.1, 0.2, 0.5, 1.0\}$.
\end{itemize}

\subsection{Simulation Experiments}

\paragraph{Experiment 1: variance inflation validation.}
For the Gaussian model with $d=1$, compute the empirical variance of $\widetilde{\theta}_{\mathrm{plug}}$ across 5000 replications and compare to the theoretical prediction $I(\theta_0)^{-1}/n + \sigma^2 I(\theta_0)^{-2}$.
Vary $n \in \{100, 500, 1000, 5000\}$ and $\eps \in \{0.1, 0.5, 1.0, 5.0, 10.0\}$.
Plot empirical vs.\ theoretical variance as a scatter plot; points should lie on the diagonal.

\paragraph{Experiment 2: coverage across the privacy spectrum.}
For each model and each $(\eps, n)$ pair, compute 95\% Wald CIs using both plug-in and noise-aware estimators.
Report coverage as a function of $\eps$ for fixed $n$ (and vice versa).
The key prediction is that coverage should be near-nominal for all methods when $n\eps^2 \gg 1$, but naive synthetic data analysis should have severely undercovered intervals especially when $\eps$ is small.

\paragraph{Experiment 3: plug-in vs.\ noise-aware under clipping.}
For logistic regression with $d=5$, vary the clipping radius $B$ from aggressive ($B=1$) to generous ($B=10$) and measure bias and coverage.
We test whether the noise-aware estimator shows less bias and better coverage when $B$ is small relative to the natural data scale.

\paragraph{Experiment 4: scaling law validation.}
For the Gaussian model, plot MSE vs.\ $n$ on a log-log scale for several values of $\eps$.
The theoretical prediction is $\mathrm{MSE} \approx c_1/n + c_2/(n^2\eps^2)$, which should appear as two distinct slopes on the log-log plot.
Identify the crossover point where the privacy term begins to dominate.

\subsection{Real Data Experiments}

\paragraph{Dataset.}
We use the American Community Survey (ACS) Public Use Microdata Sample, following the Folktables framework \citep{ding2021retiring} which provides standardized prediction tasks.
Specifically, we consider the ACSIncome task (predicting whether income exceeds \$50{,}000) for California, 2018.
After dropping rows with missing values, we standardize all features to zero mean and unit variance, then select the $d=10$ covariates with highest marginal variance, yielding a logistic regression problem.
Features are clipped to $\norm{X}_2 \le B$ with $B=3.0$ (post-standardization), giving $\ell_2$ sensitivity $\Delta_2 = 2B/n = 6/n$.
The large full sample ($n \approx 1.6\text{M}$) allows us to treat the full-data MLE as the ground truth $\theta_0$.

\paragraph{Protocol.}
Subsample datasets of size $n \in \{500, 1000, 5000, 10000\}$ and apply our pipeline at $\eps \in \{0.5, 1.0, 2.0, 5.0\}$.
For each configuration, run 500 replications.
Report coverage and interval length of Wald CIs for each regression coefficient, averaged across coordinates.

\subsection{Synthetic Data Evaluation}

Generate parametric synthetic data $D^{\mathrm{syn}}$ from $\widetilde{\theta}$ and evaluate:
\begin{itemize}
  \item Whether point estimates and CIs computed on $D^{\mathrm{syn}}$ (with noise-aware correction) match those from the DP statistic release.
  Since both are post-processing of the same $\widetilde{\bar S}$, the estimators should agree, but the CI widths may differ depending on $n_{\mathrm{syn}}$.
  \item How naive analysis (treating $D^{\mathrm{syn}}$ as real data) miscalibrates uncertainty, replicating and extending the findings of \citet{raisa2023noiseaware,perez2024mwutest}.
  We expect severe undercoverage, especially for small $\eps$.
  \item The effect of synthetic sample size $n_{\mathrm{syn}}$ on Monte Carlo error, verifying that $n_{\mathrm{syn}} \gg n$ is needed to make the synthetic sampling error negligible.
\end{itemize}

\subsection{Results}

\paragraph{Experiment 1: variance inflation validation.}
Figure~\ref{fig:variance_validation} compares the empirical variance of the DP plug-in estimator to the theoretical prediction from Theorem~\ref{thm:clt}, namely
$I(\theta_0)^{-1}/n + \sigma^2 I(\theta_0)^{-2}$.
In the Gaussian model ($d=1$), this reduces to $1/n + \sigma^2$, so each point corresponds to one $(n,\eps)$ pair and should lie on the identity line if the finite-sample behavior matches the asymptotic formula.

The points concentrate tightly around the diagonal across all privacy levels and sample sizes, with near-perfect linear agreement.
The Pearson correlation between theoretical and empirical variance is $r=0.9999978$, and the maximum relative error across the 20 $(n,\eps)$ settings is $3.67\%$.
This confirms that the additive decomposition into sampling noise and privacy noise is already accurate in finite samples, supporting the practical use of our Wald variance formula without extra calibration.

\begin{figure}[t]
  \centering
  \includegraphics[width=.85\linewidth]{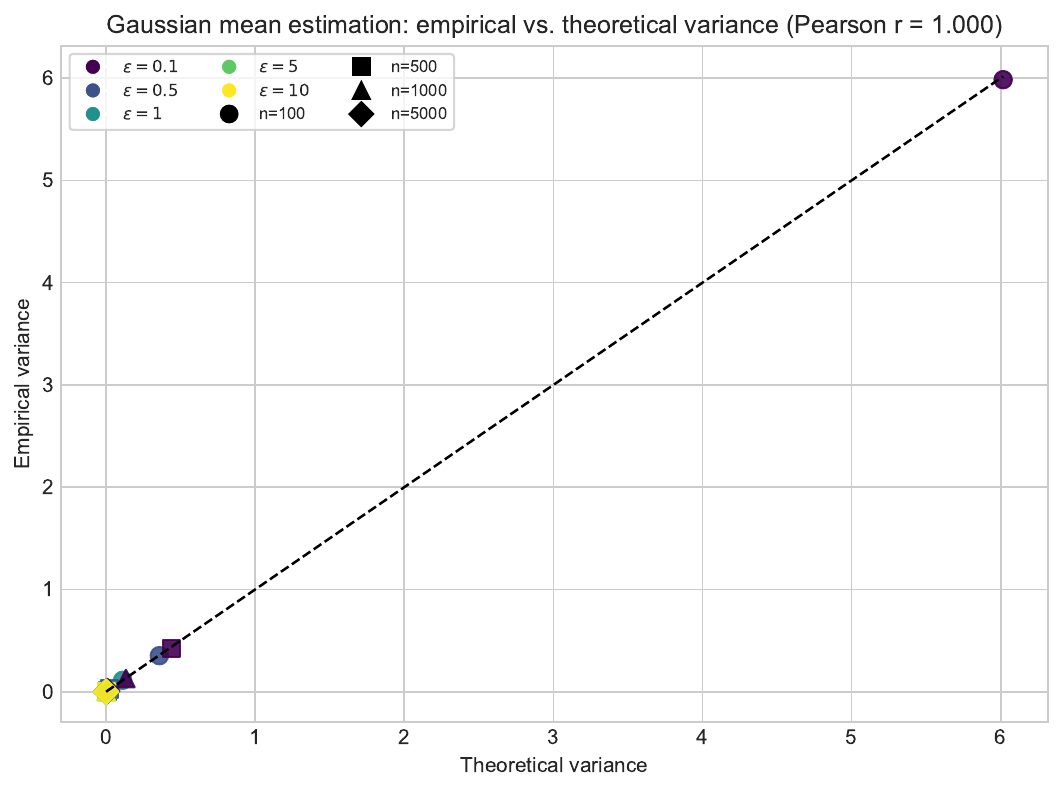}
  \caption{Empirical versus theoretical variance for the DP plug-in estimator in Gaussian mean estimation. Each point corresponds to one $(n,\eps)$ configuration, with color indicating $\eps$ and marker shape indicating $n$. The close alignment with the identity line validates the finite-sample relevance of Theorem~\ref{thm:clt}.}
  \label{fig:variance_validation}
\end{figure}

\paragraph{Experiment 2: coverage across the privacy spectrum.}
Figure~\ref{fig:coverage_vs_epsilon} reports 95\% coverage as a function of $\eps$ for Gaussian, logistic, and Poisson models at two sample sizes.
Across all models and $(n,\eps)$ combinations, non-private coverage stays in $[0.941,0.961]$, DP plug-in Wald in $[0.768,1.000]$, DP noise-aware Wald in $[0.768,1.000]$, and DP bootstrap in $[0.765,1.000]$.
The worst calibrated private setting for these three methods is the Poisson model at $(n,\eps)=(500,0.1)$, where coverage is approximately $0.77$.

The naive synthetic-data baseline severely undercovers, ranging from $0.014$ to $0.844$, with minimum coverage at Poisson $(n,\eps)=(500,0.1)$.
At $\eps=0.1$, its mean coverage across all models and $n$ is only $0.079$.
Thus Figure~\ref{fig:coverage_vs_epsilon} and Table~\ref{tab:coverage-summary} both show that accounting for privacy noise is essential for valid uncertainty quantification.

The Poisson model shows the worst calibration among the three families (coverage $\approx 0.77$ at $n{=}500, \eps{=}0.1$), reflecting a finite-sample Wald breakdown: for log-link models, $I(\theta) = X^\top \mathrm{diag}(\exp(X\theta))\, X / n$ depends on $\theta$ through the convex $\exp(\cdot)$, so by Jensen's inequality $I(\widetilde\theta)$ is systematically inflated when privacy noise is large, yielding CIs that are too narrow.
This effect vanishes as $\eps$ grows.
Poisson coverage is further hampered by the high sensitivity of $s(x,y)=yx$: with $B_X=3, B_Y=20$, the effective bound is $B=60$, producing large noise at strong privacy.

\FloatBarrier
\begin{table}[t]
\caption{Summary: 95\% CI Coverage across Methods and Privacy Levels (Gaussian Model, $n=1000$).}
\centering
\small
\setlength{\tabcolsep}{3pt}
\begin{tabular}{@{}lccccc@{}}
\toprule
Method & $\eps{=}0.1$ & $\eps{=}0.5$ & $\eps{=}1.0$ & $\eps{=}5.0$ & $\eps{=}10$ \\
\midrule
Non-private MLE      & 0.947 & 0.958 & 0.954 & 0.948 & 0.948 \\
DP plug-in Wald      & 0.950 & 0.947 & 0.951 & 0.950 & 0.955 \\
DP noise-aware Wald  & 0.950 & 0.947 & 0.951 & 0.950 & 0.955 \\
DP param.\ bootstrap & 0.944 & 0.945 & 0.945 & 0.945 & 0.948 \\
Naive DP synthetic   & 0.146 & 0.499 & 0.679 & 0.848 & 0.831 \\
\bottomrule
\end{tabular}
\label{tab:coverage-summary}
\end{table}
\FloatBarrier

\begin{figure}[t]
  \centering
  \includegraphics[width=\linewidth]{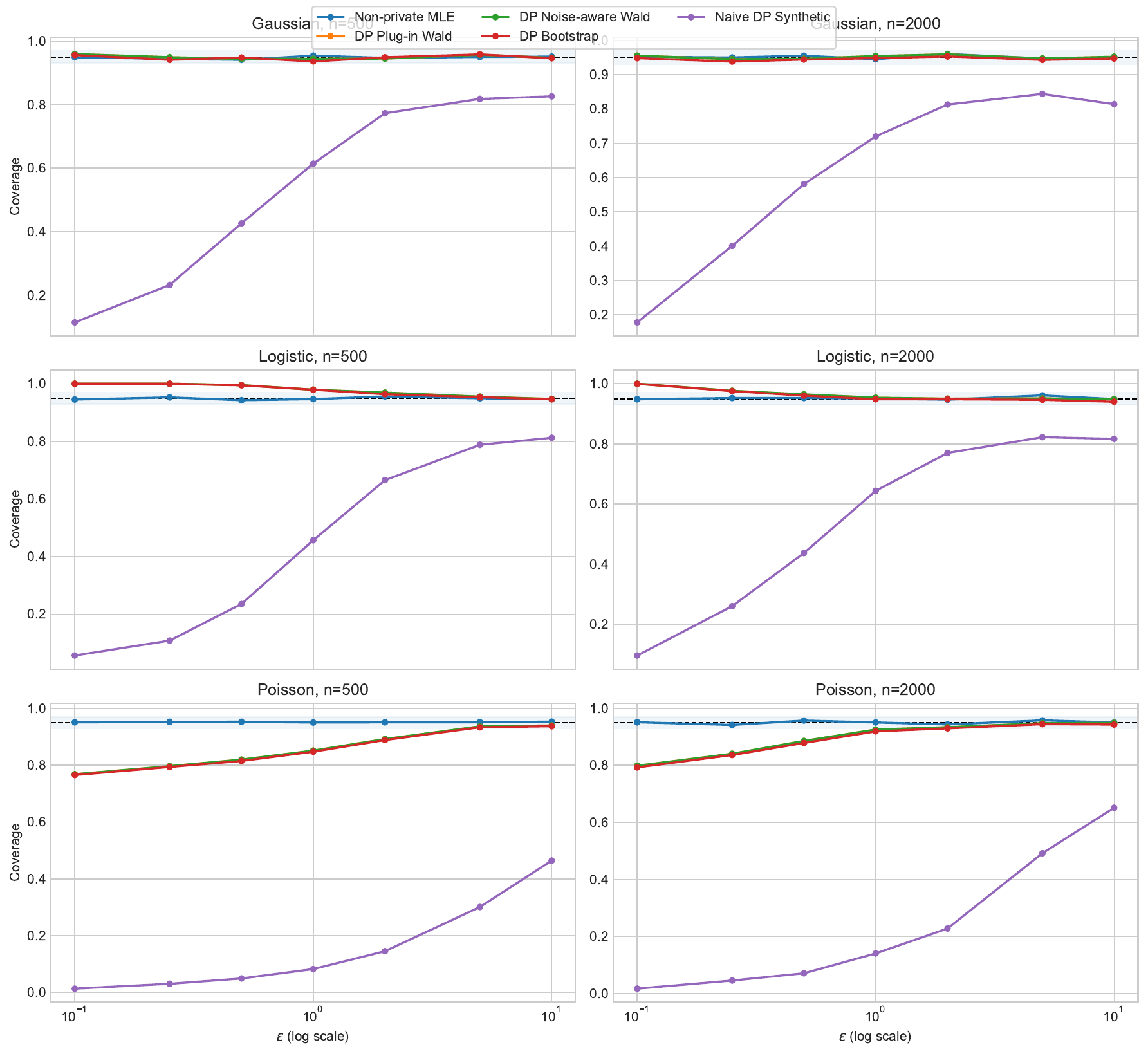}
  \caption{Coverage of 95\% intervals across privacy levels for Gaussian, logistic, and Poisson models, each at two sample sizes. The shaded band marks an acceptable calibration range around nominal coverage. Noise-calibrated DP methods remain near nominal while naive synthetic analysis undercovers in the low-$\eps$ regime.}
  \label{fig:coverage_vs_epsilon}
\end{figure}

\paragraph{Experiment 2 (precision trade-off): CI lengths.}
Figure~\ref{fig:ci_length_vs_epsilon} complements the coverage plot by reporting average CI length.
At small $\eps$, valid DP methods widen intervals to absorb privacy uncertainty, while naive synthetic analysis remains artificially narrow.
This pattern explains the simultaneous undercoverage and apparent precision of the naive baseline.

As $\eps$ grows, interval lengths for DP-aware methods contract toward the non-private benchmark, matching the asymptotic efficiency transition predicted by Corollary~\ref{cor:efficiency}.
Together, Figures~\ref{fig:coverage_vs_epsilon} and~\ref{fig:ci_length_vs_epsilon} demonstrate that honest calibration requires wider intervals precisely where privacy noise is strongest.

\begin{figure}[t]
  \centering
  \includegraphics[width=\linewidth]{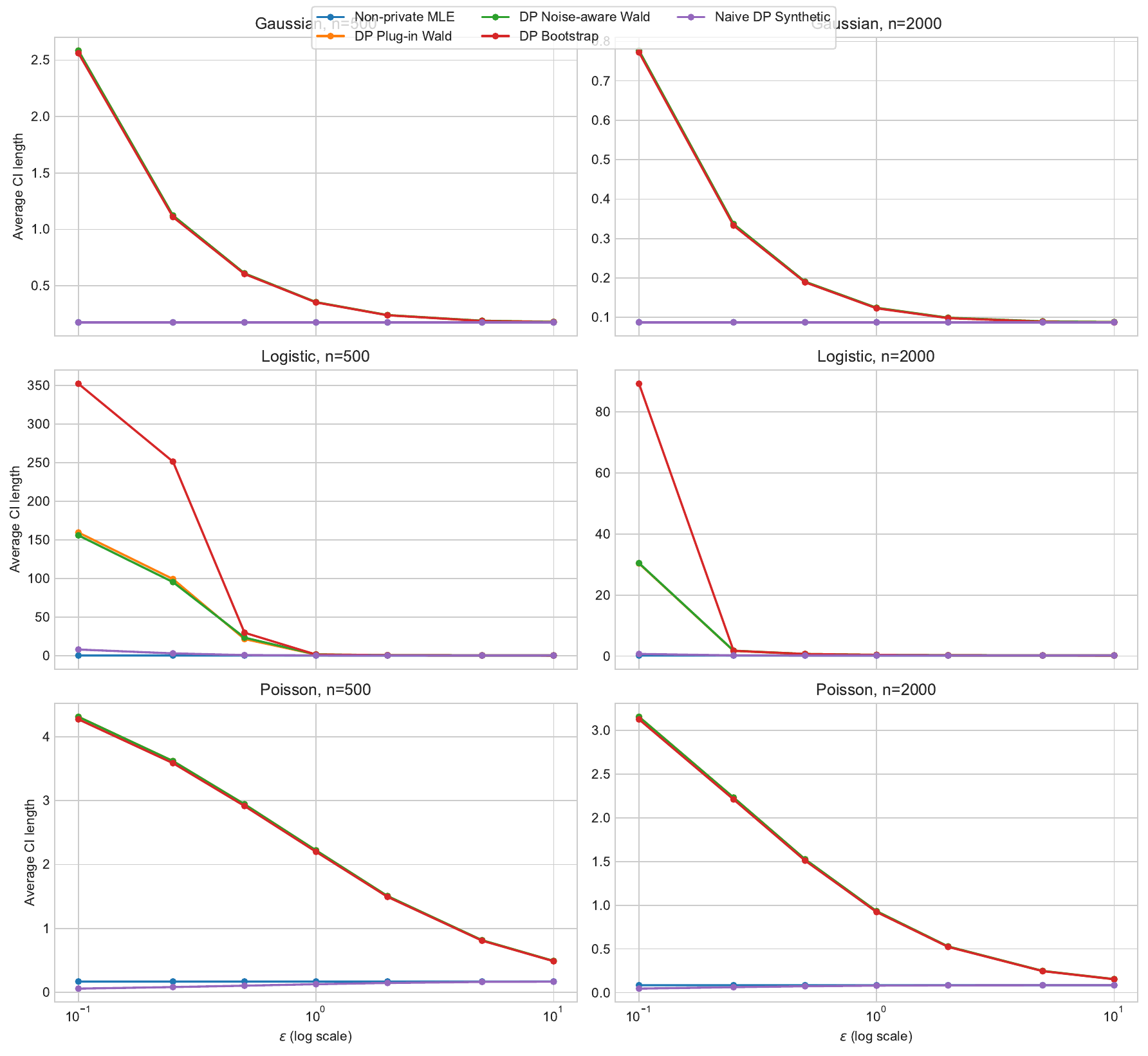}
  \caption{Average confidence-interval length versus $\eps$ for the same settings as Figure~\ref{fig:coverage_vs_epsilon}. Noise-aware methods are wider at strong privacy (small $\eps$) and contract as $\eps$ increases, reflecting the expected privacy-accuracy trade-off. Naive synthetic intervals remain narrow but are miscalibrated.}
  \label{fig:ci_length_vs_epsilon}
\end{figure}

\paragraph{Experiment 3: plug-in versus noise-aware under clipping.}
Figure~\ref{fig:clipping_comparison} isolates clipping effects in logistic regression by varying the clipping radius $B$ while fixing $(n,\eps)=(1000, 1.0)$.

The results reveal a U-shaped bias curve reflecting the sensitivity--accuracy tradeoff.
At aggressive clipping ($B \le 1.0$), clipping bias dominates: at $B=0.5$, plug-in has absolute bias $2.323$ and coverage $0.116$, while noise-aware has bias $2.767$ and coverage $0.088$.
At moderate clipping ($B=2.0$--$3.0$), both methods reduce bias to $0.16$--$0.29$ and improve coverage to $0.81$--$0.98$.
At generous clipping ($B=5.0$--$10.0$), clipping bias vanishes but $\sigma \propto 2B/n$ grows linearly with $B$, increasing error (bias $0.24$--$1.06$); coverage remains good ($0.96$--$0.97$) because Wald intervals correctly widen.

The noise-aware estimator does not outperform plug-in in any setting, consistent with Proposition~\ref{prop:equiv}: the second-order correction does not help and may hurt when the CLT approximation to clipped data is poor.
These results suggest choosing $B$ at $2$--$3$ standard deviations of the features, balancing clipping bias against sensitivity-driven noise.

\begin{figure}[t]
  \centering
  \includegraphics[width=\linewidth]{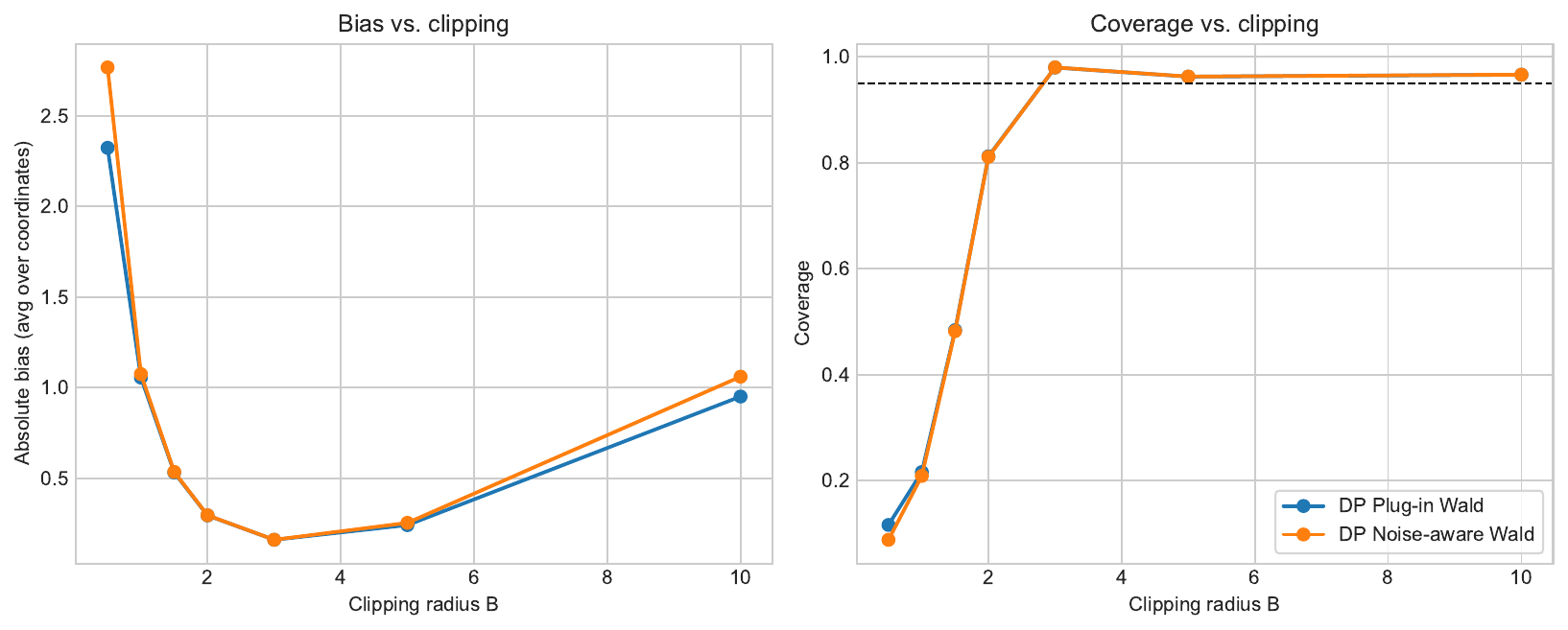}
  \caption{Logistic-regression clipping study comparing DP plug-in and DP noise-aware estimators. Left: average absolute bias versus clipping radius $B$. Right: empirical 95\% coverage versus $B$. Both methods exhibit a U-shaped bias curve: too-small $B$ causes clipping bias while too-large $B$ increases noise through higher sensitivity. The noise-aware estimator provides no advantage over plug-in, consistent with Proposition~\ref{prop:equiv}.}
  \label{fig:clipping_comparison}
\end{figure}

\paragraph{Experiment 4: scaling law validation.}
Figure~\ref{fig:scaling_law} shows log-log MSE curves for multiple $\eps$ values against sample size $n$.
The empirical curves follow the theoretical form $c_1/n + c_2/(n^2\eps^2)$, with a privacy-dominated regime at smaller $n$ and a sampling-dominated regime at larger $n$.
The first grid point where privacy variance is no larger than sampling variance is $n=100$ for $\eps=5$, $n=500$ for $\eps=2$, $n=5{,}000$ for $\eps=1$, and $n=10{,}000$ for $\eps=0.5$.
At large $n$, all DP curves approach the non-private $1/n$ baseline, confirming that privacy overhead vanishes asymptotically.
The crossover ordering matches theory: larger $\eps$ shifts the transition to smaller $n$, consistent with both Theorem~\ref{thm:clt} and Theorem~\ref{thm:lb}.

\begin{figure}[t]
  \centering
  \includegraphics[width=.9\linewidth]{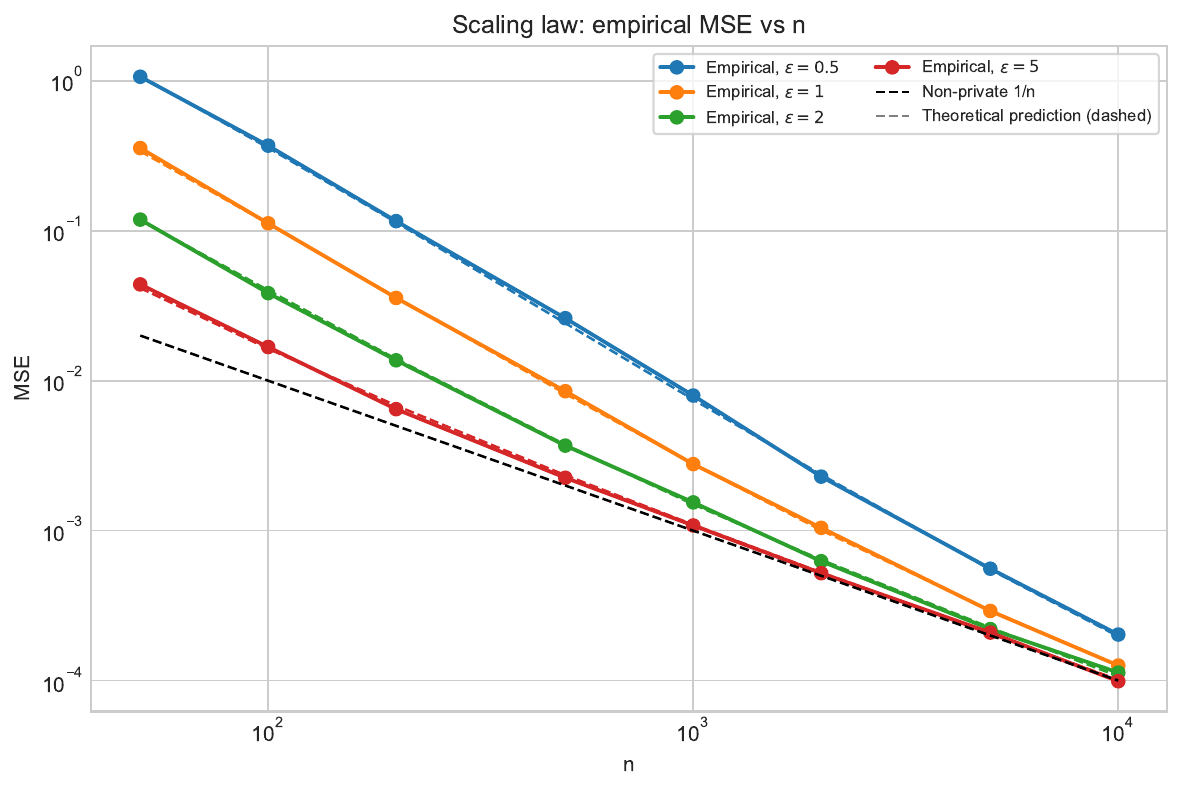}
  \caption{Log-log MSE scaling for Gaussian mean estimation. Solid lines are empirical MSE for each $\eps$; dashed color-matched lines are theoretical predictions; the thin black dashed line is non-private $1/n$. The observed two-regime behavior matches the predicted $1/n + 1/(n^2\eps^2)$ structure.}
  \label{fig:scaling_law}
\end{figure}

\paragraph{Experiment 5: type-I error and power.}
Figure~\ref{fig:power_curves} reports rejection rates for Wald tests across alternatives $\Delta$ and privacy levels.
At $\Delta=0$ and $\eps=1.0$, type-I error is $0.0385$ for DP plug-in/noise-aware, $0.0515$ for non-private, and $0.292$ for naive synthetic.
Across privacy levels, naive synthetic remains anti-conservative under the null (type-I error from $0.488$ at $\eps=0.5$ to $0.184$ at $\eps=5.0$).

As $\Delta$ increases, all methods gain power; for example at $\eps=1.0$ and $\Delta=0.1$, DP power is $0.481$ versus $0.873$ non-private.
By $\Delta=0.2$, DP power reaches $0.967$ and non-private reaches $1.000$.
The expected privacy-power trade-off is visible at lower $\eps$, where calibrated methods pay power to maintain valid size.
Thus, proper noise calibration delivers valid tests while preserving competitive sensitivity to true effects.

\begin{figure}[t]
  \centering
  \includegraphics[width=\linewidth]{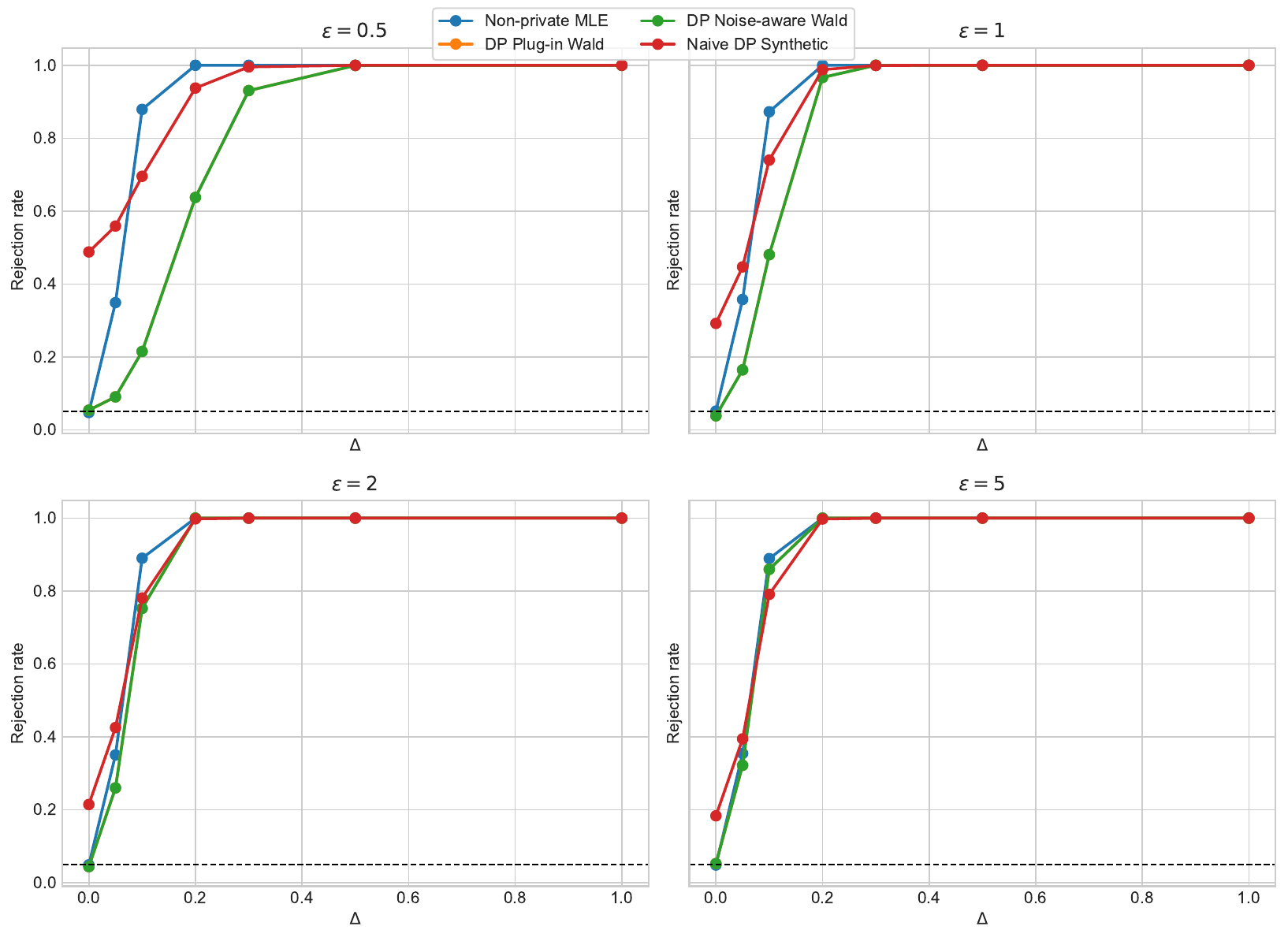}
  \caption{Power curves and type-I error diagnostics for Gaussian Wald tests across four privacy levels. At $\Delta=0$, calibrated DP methods remain near the nominal 0.05 level while naive synthetic testing is inflated. Power increases with effect size and approaches non-private behavior as $\eps$ grows.}
  \label{fig:power_curves}
\end{figure}

\paragraph{Experiment 6: real data (ACS/Folktables).}
Figure~\ref{fig:real_data_coverage} summarizes ACSIncome logistic-regression results over $(n,\eps)$ grids using the full-data MLE as ground truth.
Across all $(n,\eps)$ settings, mean coverage is $0.889$ for DP plug-in, $0.888$ for DP noise-aware, $0.880$ for DP bootstrap, and $0.510$ for naive synthetic.
The naive method is worst at $(n,\eps)=(500,0.5)$ with coverage $0.134$, while the private calibrated methods at that setting are all near $1.0$.

After the numerical stabilization in Section~\ref{sec:noiseaware}, CI lengths remain finite in all settings:
for DP plug-in they range from $0.127$ to $90.835$, for DP noise-aware from $0.127$ to $93.912$, and for DP bootstrap from $0.125$ to $137.185$.
Naive synthetic intervals are much shorter ($0.118$ to $2.735$) but substantially undercovered.
These results indicate that the sufficient-statistic release pipeline is practical beyond idealized synthetic settings.

\begin{figure}[t]
  \centering
  \includegraphics[width=\linewidth]{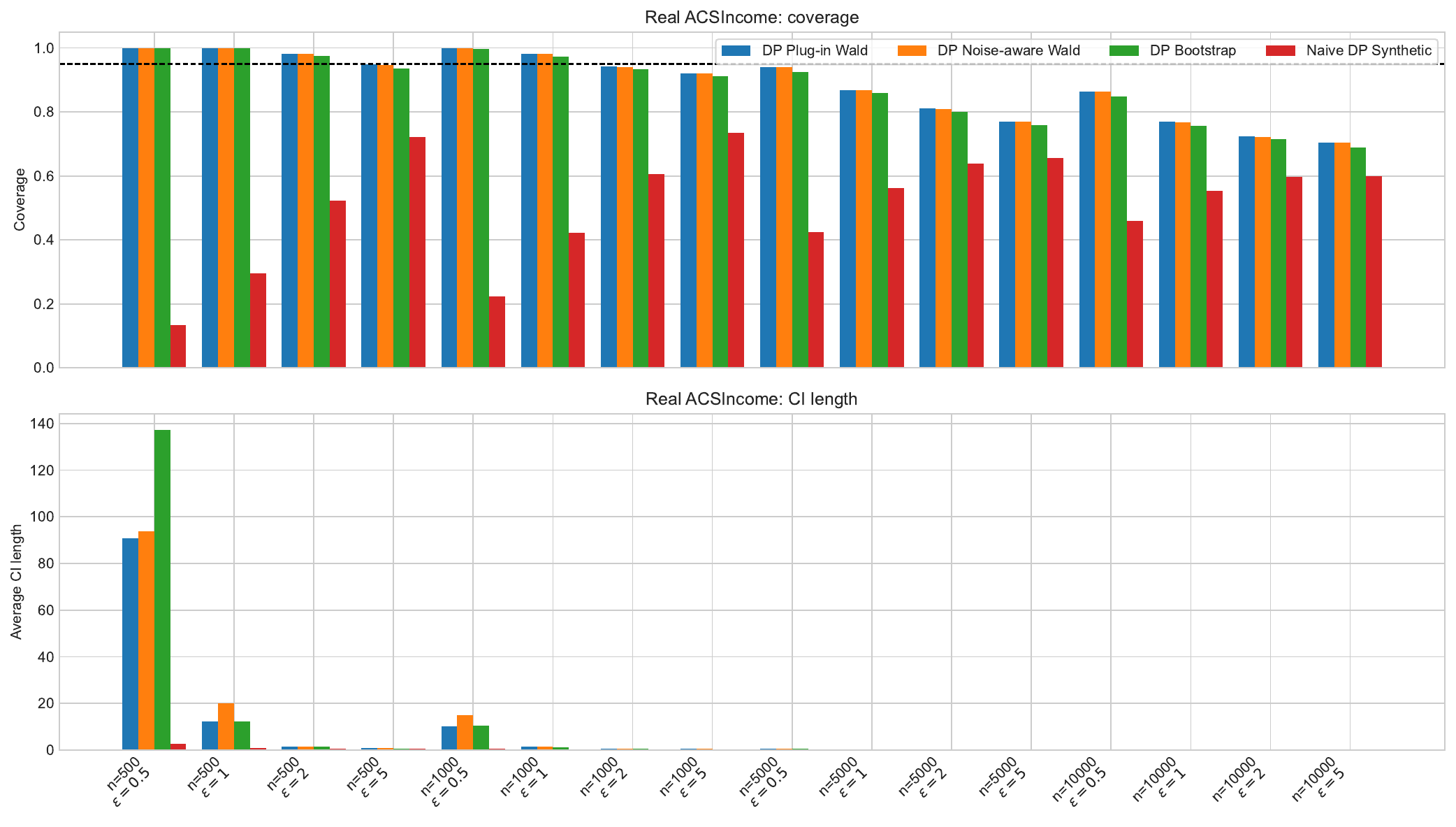}
  \caption{Real-data ACSIncome evaluation across $(n,\eps)$ configurations. The top panel reports average coefficient coverage and the bottom panel reports average CI length. DP plug-in and noise-aware methods maintain substantially better calibration than naive synthetic analysis, with the largest gains in low-$n$, low-$\eps$ settings.}
  \label{fig:real_data_coverage}
\end{figure}

\paragraph{Experiment 7: synthetic-data inferential evaluation.}
Figure~\ref{fig:synthetic_eval} compares three analysis modes as synthetic sample size $n_{\mathrm{syn}}$ grows.
Direct inference from $\widetilde{\theta}$ has coverage $0.952$ at all ratios, while noise-aware synthetic analysis is $0.942,0.943,0.954,0.957$ for $n_{\mathrm{syn}}/n=1,5,10,50$.
Naive synthetic analysis is severely undercovered at every ratio: $0.685,0.391,0.282,0.136$.

Increasing $n_{\mathrm{syn}}$ reduces only synthetic Monte Carlo noise, not the privacy distortion inherited from $\widetilde{\theta}$.
Thus synthetic sample size alone cannot fix inferential miscalibration caused by ignoring privacy noise.

\begin{figure}[t]
  \centering
  \includegraphics[width=.85\linewidth]{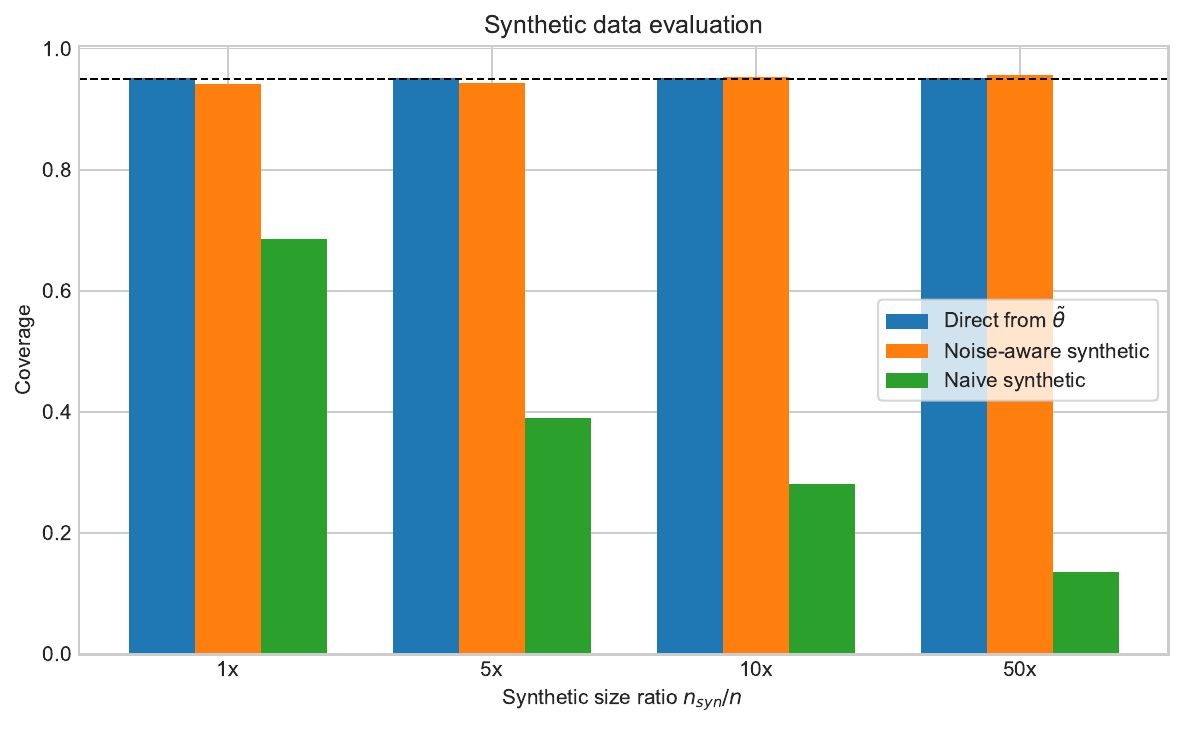}
  \caption{Coverage comparison for direct DP inference, noise-aware synthetic analysis, and naive synthetic analysis across synthetic sample-size ratios. Direct and noise-aware approaches remain near nominal coverage, while naive analysis systematically undercovers. Larger synthetic datasets do not eliminate the need for privacy-aware uncertainty correction.}
  \label{fig:synthetic_eval}
\end{figure}

\section{Discussion and Extensions}\label{sec:discussion}

\paragraph{Experimental takeaways.}
Across all experiments, the variance inflation formula from Theorem~\ref{thm:clt} is highly accurate, calibrated DP inference strongly outperforms naive synthetic-data analysis, and plug-in and noise-aware estimators behave similarly outside extreme finite-sample regimes, consistent with Proposition~\ref{prop:equiv}.

\paragraph{Beyond exponential families.}
The sufficient-statistic release idea extends to M-estimation and the generalized method of moments, where ``sufficient statistics'' become empirical moment conditions $n^{-1}\sum_i \psi(X_i;\theta)$ that can be privatized and used for noise-aware inference, connecting naturally to workload-based synthesis and causal inference pipelines.

\paragraph{Noise shaping and matrix mechanisms.}
For high-dimensional $d$, isotropic Gaussian noise may be suboptimal.
Matrix mechanisms \citep{li2015matrix} can reduce variance in target directions by releasing a linear transform $M\bar S$ with noise calibrated to the sensitivity of $M\bar S$; extending our theory to this setting is straightforward, requiring only a modified sensitivity analysis for the transformed query $M\bar S$.

\paragraph{Composition with DP-SGD.}
When $B$ is unknown a priori and data-dependent clipping is used (as in differentially private stochastic gradient descent, DP-SGD), the sensitivity analysis becomes more involved; investigating this interaction is an important practical direction.

\paragraph{Multivariate confidence regions.}
Our Wald intervals are coordinate-wise; simultaneous ellipsoidal regions can be constructed from the same variance formula and are left for future work.

\paragraph{Limitations.}
Our pipeline requires a regular exponential family with known sufficient statistics; extending to semiparametric or nonparametric models needs different tools.
Wald CIs can undercover at low $\eps$ for convex-link models (e.g., Poisson) due to Jensen's inequality inflating $I(\widetilde\theta)$; bootstrap CIs may help.
The clipping radius $B$ trades off clipping bias against noise magnitude, closely related to bias--accuracy--privacy trilemmas \citep{kamath2025trilemma}.

\begin{acknowledgements}
\end{acknowledgements}

\bibliography{references}

@inproceedings{dwork2006calibrating,
  title        = {Calibrating Noise to Sensitivity in Private Data Analysis},
  author       = {Dwork, Cynthia and McSherry, Frank and Nissim, Kobbi and Smith, Adam},
  booktitle    = {Theory of Cryptography Conference (TCC)},
  pages        = {265--284},
  year         = {2006},
  publisher    = {Springer}
}

@book{dwork2014foundations,
  title        = {The Algorithmic Foundations of Differential Privacy},
  author       = {Dwork, Cynthia and Roth, Aaron},
  year         = {2014},
  publisher    = {Now Publishers},
  series       = {Foundations and Trends in Theoretical Computer Science},
  volume       = {9},
  pages        = {211--407}
}

@inproceedings{balle2018improving,
  title        = {Improving the {G}aussian Mechanism for Differential Privacy: Analytical Calibration and Optimal Denoising},
  author       = {Balle, Borja and Wang, Yu{-}Xiang},
  booktitle    = {Proceedings of the 35th International Conference on Machine Learning (ICML)},
  series       = {Proceedings of Machine Learning Research},
  volume       = {80},
  pages        = {403--412},
  year         = {2018},
  publisher    = {PMLR},
}

@article{mckenna2021nist,
  title        = {Winning the {NIST} Contest: A Scalable and General Approach to Differentially Private Synthetic Data},
  author       = {McKenna, Ryan and Miklau, Gerome and Sheldon, Daniel},
  journal      = {Journal of Privacy and Confidentiality},
  volume       = {11},
  number       = {3},
  year         = {2021}
}

@article{tao2022benchmark,
  title        = {Benchmarking Differentially Private Synthetic Data Generation Algorithms},
  author       = {Tao, Yuchao and McKenna, Ryan and Hay, Michael and Machanavajjhala, Ashwin and Miklau, Gerome},
  journal      = {arXiv preprint arXiv:2112.09238},
  year         = {2022}
}

@article{mckenna2022aim,
  title        = {{AIM}: An Adaptive and Iterative Mechanism for Differentially Private Synthetic Data},
  author       = {McKenna, Ryan and Mullins, Brett and Sheldon, Daniel and Miklau, Gerome},
  journal      = {Proceedings of the VLDB Endowment},
  volume       = {15},
  number       = {11},
  pages        = {2599--2612},
  year         = {2022},
  doi          = {10.14778/3551793.3551817}
}

@inproceedings{raisa2023noiseaware,
  title        = {Noise-Aware Statistical Inference with Differentially Private Synthetic Data},
  author       = {R{\"a}is{\"a}, Ossi and J{\"a}lk{\"o}, Joonas and Kaski, Samuel and Honkela, Antti},
  booktitle    = {Proceedings of The 26th International Conference on Artificial Intelligence and Statistics (AISTATS)},
  series       = {Proceedings of Machine Learning Research},
  volume       = {206},
  pages        = {3620--3643},
  year         = {2023},
  publisher    = {PMLR}
}

@inproceedings{ferrando2022bootstrap,
  title        = {Parametric Bootstrap for Differentially Private Confidence Intervals},
  author       = {Ferrando, Cecilia and Wang, Shufan and Sheldon, Daniel},
  booktitle    = {Proceedings of The 25th International Conference on Artificial Intelligence and Statistics (AISTATS)},
  series       = {Proceedings of Machine Learning Research},
  volume       = {151},
  pages        = {1598--1618},
  year         = {2022},
  publisher    = {PMLR}
}

@inproceedings{alabi2022hypothesis,
  title        = {Hypothesis Testing for Differentially Private Linear Regression},
  author       = {Alabi, Daniel and Vadhan, Salil},
  booktitle    = {Advances in Neural Information Processing Systems},
  year         = {2022}
}

@inproceedings{bernstein2018dpbayes,
  title        = {Differentially Private Bayesian Inference for Exponential Families},
  author       = {Bernstein, Garrett and Sheldon, Daniel},
  booktitle    = {Advances in Neural Information Processing Systems},
  volume       = {31},
  pages        = {2924--2934},
  year         = {2018}
}

@inproceedings{donhauser2024privpgd,
  title        = {Privacy-Preserving Data Release Leveraging Optimal Transport and Particle Gradient Descent},
  author       = {Donhauser, Konstantin and Abad, Javier and Hulkund, Neha and Yang, Fanny},
  booktitle    = {Proceedings of the 41st International Conference on Machine Learning (ICML)},
  year         = {2024}
}

@article{perez2024mwutest,
  title        = {Does Differentially Private Synthetic Data Lead to Synthetic Discoveries?},
  author       = {Perez, Ileana Montoya and Movahedi, Parisa and Nieminen, Valtteri and Airola, Antti and Pahikkala, Tapio},
  journal      = {Methods of Information in Medicine},
  year         = {2024},
  doi          = {10.1055/a-2385-1355},
  note         = {Published online Sep 9, 2024}
}

@article{wang2025dpbootstrap,
  title        = {Differentially Private Bootstrap: New Privacy Analysis and Inference Strategies},
  author       = {Wang, Zhanyu and Cheng, Guang and Awan, Jordan},
  journal      = {Journal of Machine Learning Research},
  volume       = {26},
  pages        = {1--57},
  year         = {2025}
}

@article{covington2025unbiased,
  title        = {Unbiased Statistical Estimation and Valid Confidence Intervals under Differential Privacy},
  author       = {Covington, Christian and He, Xi and Honaker, James and Kamath, Gautam},
  journal      = {Statistica Sinica},
  volume       = {35},
  pages        = {651--670},
  year         = {2025}
}

@inproceedings{smith2011privacy,
  title        = {Privacy-Preserving Statistical Estimation with Optimal Convergence Rates},
  author       = {Smith, Adam},
  booktitle    = {Proceedings of the 43rd Annual ACM Symposium on Theory of Computing (STOC)},
  pages        = {813--821},
  year         = {2011},
  publisher    = {ACM}
}

@article{duchi2018minimax,
  title        = {Minimax Optimal Procedures for Locally Private Estimation},
  author       = {Duchi, John C. and Jordan, Michael I. and Wainwright, Martin J.},
  journal      = {Journal of the American Statistical Association},
  volume       = {113},
  number       = {521},
  pages        = {182--201},
  year         = {2018}
}

@article{cai2021cost,
  title        = {The Cost of Privacy: Optimal Rates of Convergence for Parameter Estimation with Differential Privacy},
  author       = {Cai, T. Tony and Wang, Yichen and Zhang, Linjun},
  journal      = {Annals of Statistics},
  volume       = {49},
  number       = {5},
  pages        = {2825--2850},
  year         = {2021}
}

@article{wasserman2010statistical,
  title        = {A Statistical Framework for Differential Privacy},
  author       = {Wasserman, Larry and Zhou, Shuheng},
  journal      = {Journal of the American Statistical Association},
  volume       = {105},
  number       = {489},
  pages        = {375--389},
  year         = {2010},
  doi          = {10.1198/jasa.2009.tm08651}
}

@inproceedings{karwa2018finite,
  title        = {Finite Sample Differentially Private Confidence Intervals},
  author       = {Karwa, Vishesh and Vadhan, Salil P.},
  booktitle    = {9th Innovations in Theoretical Computer Science Conference (ITCS)},
  series       = {LIPIcs},
  volume       = {94},
  pages        = {44:1--44:9},
  year         = {2018},
  publisher    = {Schloss Dagstuhl - Leibniz-Zentrum f{\"u}r Informatik},
  doi          = {10.4230/LIPIcs.ITCS.2018.44}
}

@article{karwa2016noisydegrees,
  title        = {Inference using noisy degrees: Differentially private {$\\beta$}-model and synthetic graphs},
  author       = {Karwa, Vishesh and Slavkovi{\'c}, Aleksandra},
  journal      = {The Annals of Statistics},
  volume       = {44},
  number       = {1},
  year         = {2016},
  doi          = {10.1214/15-AOS1358}
}

@article{avellamedina2023noisyopt,
  title        = {Differentially private inference via noisy optimization},
  author       = {Avella-Medina, Marco and Bradshaw, James and Loh, Po-Ling},
  journal      = {The Annals of Statistics},
  volume       = {51},
  number       = {5},
  year         = {2023},
  doi          = {10.1214/23-AOS2321}
}

@article{dong2022gdp,
  title        = {Gaussian Differential Privacy},
  author       = {Dong, Jinshuo and Roth, Aaron and Su, Weijie J.},
  journal      = {Journal of the Royal Statistical Society Series B: Statistical Methodology},
  volume       = {84},
  number       = {1},
  pages        = {3--37},
  year         = {2022},
  doi          = {10.1111/rssb.12454}
}

@article{awan2025simbainference,
  title        = {Simulation-Based, Finite-Sample Inference for Privatized Data},
  author       = {Awan, Jordan and Wang, Yuguo},
  journal      = {Journal of the American Statistical Association},
  volume       = {120},
  number       = {551},
  pages        = {1669--1682},
  year         = {2025},
  doi          = {10.1080/01621459.2024.2427436}
}

@article{kamath2025trilemma,
  title        = {A Bias-Accuracy-Privacy Trilemma for Statistical Estimation},
  author       = {Kamath, Gautam and Mouzakis, Argyris and Regehr, Matthew and Singhal, Vatsal},
  journal      = {Journal of the American Statistical Association},
  volume       = {120},
  number       = {552},
  pages        = {2338--2349},
  year         = {2025},
  doi          = {10.1080/01621459.2024.2443275}
}

@inproceedings{awan2018dpuf,
  title        = {Differentially Private Uniformly Most Powerful Tests for Binomial Data},
  author       = {Awan, Jordan and Slavkovi{\'c}, Aleksandra},
  booktitle    = {Advances in Neural Information Processing Systems},
  volume       = {31},
  year         = {2018}
}

@inproceedings{ding2021retiring,
  title        = {Retiring Adult: New Datasets for Fair Machine Learning},
  author       = {Ding, Frances and Hardt, Moritz and Miller, John and Schmidt, Ludwig},
  booktitle    = {Advances in Neural Information Processing Systems},
  volume       = {34},
  pages        = {6478--6490},
  year         = {2021}
}

@article{li2015matrix,
  title        = {The Matrix Mechanism: Optimizing Linear Counting Queries under Differential Privacy},
  author       = {Li, Chao and Miklau, Gerome and Hay, Michael and McGregor, Andrew and Rastogi, Vibhor},
  journal      = {The VLDB Journal},
  volume       = {24},
  number       = {6},
  pages        = {757--781},
  year         = {2015}
}

\newpage
\onecolumn

\title{Noise-Calibrated Inference from Differentially Private Sufficient Statistics in Exponential Families\\(Supplementary Material)}
\maketitle

\appendix

\vspace{-1em}
\section{Proof of Theorem~\ref{thm:clt} (Asymptotic Distribution and Variance Inflation)}
\label{app:proof_thm2}

\begin{proof}
Recall that $\mu(\theta):=\nabla A(\theta)=\E_\theta[s(X)]$ is the mean parameter map, and $(\nabla A)^{-1}$ is its inverse.
We decompose:
\[
\widetilde{\theta}_{\mathrm{plug}} - \theta_0
= (\nabla A)^{-1}(\widetilde{\bar S}) - \theta_0
= (\nabla A)^{-1}(\bar S + Z) - \theta_0.
\]
\emph{Step 1 (sampling term).}
By the law of large numbers, $\bar S \pto \mu(\theta_0)$.
By the multivariate CLT applied to $s(X_1),\ldots,s(X_n)$,
\[
\sqrt{n}\big(\bar S - \mu(\theta_0)\big) \dto \N\big(0, \Cov_{\theta_0}[s(X)]\big) = \N(0, I(\theta_0)),
\]
where the last equality uses the identity $\Cov_{\theta_0}[s(X)] = \nabla^2 A(\theta_0) = I(\theta_0)$ for minimal exponential families.
Since $(\nabla A)^{-1}$ is differentiable at $\mu(\theta_0)$ with Jacobian $I(\theta_0)^{-1}$, the delta method gives
$\sqrt{n}(\hat\theta - \theta_0) \dto \N(0, I(\theta_0)^{-1})$.

\emph{Step 2 (privacy term).}
By a first-order Taylor expansion of $(\nabla A)^{-1}$ around $\bar S$,
\[
\widetilde{\theta}_{\mathrm{plug}} - \hat\theta
= (\nabla A)^{-1}(\bar S + Z) - (\nabla A)^{-1}(\bar S)
= I(\hat\theta)^{-1} Z + O_p(\norm{Z}^2).
\]
Since $Z\sim\N(0,\sigma^2 I_d)$, we have $\norm{Z}=O_p(\sigma)$ and thus the remainder term is $O_p(\sigma^2)=o_p(\sigma)$.
Because $I(\hat\theta) \pto I(\theta_0)$, we have $\widetilde{\theta}_{\mathrm{plug}} - \hat\theta = I(\theta_0)^{-1} Z + o_p(\sigma)$.

\emph{Step 3 (combining).}
The sampling term $\hat\theta - \theta_0$ depends only on $\{X_i\}$, while $Z$ is independent of the data.
Therefore the covariance of the two terms is zero, and
\[
\sqrt{n}(\widetilde{\theta}_{\mathrm{plug}} - \theta_0)
\dto \N\!\Big(0,\; I(\theta_0)^{-1} + n\sigma^2 I(\theta_0)^{-2}\Big).
\]
The rate decomposition follows from the fact that, under Gaussian-mechanism calibration, $\sigma$ is proportional to the sensitivity $\Delta_2=2B/n$ (up to $\log(1/\del)$ factors), yielding the $O_p((n\eps)^{-1})$ privacy distortion rate for bounded problems.
\end{proof}

\section{Proof of Proposition~\ref{prop:equiv} (First-Order Equivalence)}
\label{app:proof_prop1}

\begin{proof}
The plug-in estimator solves $\nabla A(\theta) = \widetilde{\bar S}$, while the noise-aware estimator minimizes the weighted quadratic
$Q(\theta) = (\widetilde{\bar S} - \nabla A(\theta))^\top \Sigma(\theta)^{-1} (\widetilde{\bar S} - \nabla A(\theta))$
with $\Sigma(\theta) = I(\theta)/n + \sigma^2 I$.
The noise-aware first-order condition is
$\nabla A(\theta)^\top [\nabla^2 A(\theta)] \Sigma(\theta)^{-1} (\widetilde{\bar S} - \nabla A(\theta)) = 0$.
At the plug-in solution, $\widetilde{\bar S} - \nabla A(\widetilde{\theta}_{\mathrm{plug}}) = 0$, which also sets this expression to zero.
Hence both estimators solve the same equation at leading order.

More precisely, expand $Q(\theta)$ around $\widetilde{\theta}_{\mathrm{plug}}$.
The difference $\widetilde{\theta}_{\mathrm{NA}} - \widetilde{\theta}_{\mathrm{plug}}$ satisfies a Newton-type update whose magnitude is controlled by
$\norm{\Sigma(\widetilde{\theta}_{\mathrm{plug}})^{-1} - \Sigma(\theta_0)^{-1}} = O_p(n^{-1/2} + \sigma)$
times lower-order terms.
By standard M-estimation arguments, this yields $\widetilde{\theta}_{\mathrm{NA}} - \widetilde{\theta}_{\mathrm{plug}} = o_p(n^{-1/2}) + o_p(\sigma)$.
\end{proof}

\section{Proof of Theorem~\ref{thm:lb} (Lower Bound)}
\label{app:proof_thm3}

\begin{proof}
Let $X\in\{-1,+1\}$ and write the mean parameter as $\mu(\theta)=\E_\theta[X]=\tanh(\theta)$.
Let $\widehat\theta$ be any $(\eps,\del)$-DP estimator of $\theta$ based on $n$ i.i.d.\ samples and define the post-processed mean estimator $\widehat\mu:=\tanh(\widehat\theta)$.
By post-processing invariance, $\widehat\mu$ is also $(\eps,\del)$-DP.
Since $\tanh$ is 1-Lipschitz, we have the pointwise inequality $(\widehat\mu-\mu(\theta))^2 \le (\widehat\theta-\theta)^2$ and therefore
\[
\E_\theta\!\left[(\widehat\theta-\theta)^2\right]
\;\ge\;
\E_\theta\!\left[(\widehat\mu-\mu(\theta))^2\right].
\]

It therefore suffices to lower bound the minimax mean-estimation risk under central $(\eps,\del)$-DP.
The minimax lower bounds of \citet{cai2021cost} for bounded mean estimation imply that, for $\eps\in(0,1]$ and $\del<1/4$,
\[
\inf_{\widehat{\mu}\ \text{$(\eps,\del)$-DP}} \ \sup_{\mu\in[-1,1]}
\E\!\left[(\widehat{\mu}-\mu)^2\right]
\;\ge\;
\frac{c}{n^2\eps^2}
\]
for a universal constant $c>0$.
Applying this bound to our Bernoulli family and combining with the inequality above yields the stated lower bound for $\theta$.
\end{proof}

\section{Additional Experimental Details}
\label{app:experiment_details}

\paragraph{Replication counts.}
All simulation experiments use sufficient replications for stable Monte Carlo estimates: Experiment~1 uses 5{,}000 replications, Experiments~2--3 use 1{,}000, Experiments~4--5 use 2{,}000, Experiment~6 (real data) uses 500, and Experiment~7 uses 1{,}000.
The DP parametric bootstrap uses $B_{\mathrm{boot}}=500$ draws in simulation experiments and $B_{\mathrm{boot}}=200$ in the real-data experiment.

\paragraph{Runtime.}
On a single machine with 8 CPU cores, the full experiment suite (Experiments~1--7 with all models and configurations) runs in approximately 4--6 hours.
Individual experiments range from 2 minutes (Experiment~1, Gaussian only) to 90 minutes (Experiment~6, real data with 500 replications across 16 $(n,\eps)$ configurations).
The noise-aware estimator adds negligible overhead compared to plug-in, since both require a single optimization step per replication.

\paragraph{Software.}
All experiments are implemented in Python using NumPy and SciPy.
The ACS data in Experiment~6 are accessed via the Folktables package \citep{ding2021retiring}.

\end{document}